\documentclass[twoside]{article}
\usepackage[accepted]{aistats2020}

\usepackage[round]{natbib}

\usepackage{amsfonts,amssymb,amsmath,amsbsy}
\usepackage{graphicx}
\usepackage{enumerate}
\usepackage{lscape}
\usepackage{longtable}
\usepackage{rotating}
\usepackage{color}
\usepackage{url}
\usepackage{subfigure}
\usepackage{rotating}
\usepackage[title,titletoc,toc]{appendix}
\usepackage{mathtools}
\usepackage{algorithm}
\usepackage{algpseudocode}
\usepackage{fullpage}
\usepackage[round]{natbib}
\usepackage{xspace}
\usepackage{tikz}
\usepackage{cleveref}
\usepackage{bm}
\usepackage{amsthm}

\def\R{\mathbb{R}}

\newcommand{\x}{\mathbf{x}}
\newcommand{\X}{\mathbf{X}}
\newcommand{\y}{\mathbf{y}}
\newcommand{\f}{\mathbf{f}}

\newcommand{\K}{\mathbf{K}}

\newcommand{\s}{\mathbf{s}}

\newcommand{\Rset}{\mathbb{R}}
\newcommand{\Xset}{\mathbb{X}}
\newcommand{\Nset}{\mathbb{N}}
\newcommand{\Sset}{\mathbb{S}}

\newcommand{\rank}{{\text{rank}}}
\newcommand{\esp}{\mathbb{E}}

\newcommand{\+}{_{n+1}}

\newcommand{\nlaw}{\mathcal{N}}

\newcommand{\MS}{\boldsymbol{\Sigma}_\mathbf{ff}}

\newcommand{\vm}{{\boldsymbol{\mu}_\mathbf{f}}}

\newcommand{\KL}{\operatorname{KL}}
\def\nn{\nonumber}

\newtheorem{lemma}{Lemma}
\newtheorem{theorem}{Theorem}

\newtheorem{assumption}{Assumption}
\begin{document}

\twocolumn[

\aistatstitle{Ordinal Bayesian Optimisation}

\aistatsauthor{ Victor Picheny \And Sattar Vakili \And  Artem Artemev }

\aistatsaddress{ Prowler.io\\Cambridge, UK\\
\{victor, sattar, artem\}@prowler.io } ]

\begin{abstract}
Bayesian optimisation is a powerful tool to solve expensive black-box problems,
but fails when the stationary assumption made on the objective function is strongly violated,
which is the case in particular for ill-conditioned or discontinuous objectives.
We tackle this problem by proposing a new Bayesian optimisation framework that only considers the ordering of variables, both in the input and output spaces, to fit a Gaussian process in a latent space. By doing so, our approach is agnostic to the original metrics on the original spaces. We propose two algorithms, respectively based on an optimistic strategy and on Thompson sampling. For the optimistic strategy we prove an optimal performance under the measure of regret in the latent space. We illustrate the capability of our framework on several challenging toy problems.
\end{abstract}

%%%%%%%%%%%%%%%%%%%%%%%%%%%%%%%%%%%%%%%%%%%%%%%%%%%%%%%%%%%%%%%%%%%%%%%%%%%%%%%%%%%%%%%%%%%%%%%%%%%%%%%%%%%%
\section{Introduction}
We address typical Bayesian optimisation (BO) problems, of the form:
\begin{eqnarray*}
 \min_{x \in \Xset}  g(\x),
\end{eqnarray*}
with $\Xset \in \Rset^d$ is usually a bounded hyperrectangle, $g:\Rset^d \rightarrow \Rset$ is a scalar-valued objective function, available only through noisy observations $y_i = g(\x_i)+\epsilon_i$.

BO is established as a strong competitor among derivative-free optimisation approaches, in particular for computationally expensive (low data regime) problems. 
In BO, non-parametric Gaussian processes (GPs) provide flexible and fast-to-evaluate surrogates of the objective functions. Sequential design decisions, so-called acquisitions, judiciously balance exploration and exploitation in search for global optima, leveraging the uncertainty estimates provided by the GP posterior distributions 
(see \citet{mockus1978application,jones1998efficient} for early works or 
\citet{shahriari2015taking} for a recent review).

One of the weaknesses of vanilla BO lies in the underlying assumption that the objective function is a realisation of a GP: when this assumption is strongly violated, the GP model is weakly predictive and BO becomes inefficient. Two classical examples where BO fails are
ill-conditioned problems, when the objective function has strong variations on the domain boundaries but is very flat in its central region (or conversely), and non-Lipschitz objectives, for instance with local discontinuities. High conditioning is typical in ``exploratory'' optimisation problems, when the parameter space is initially chosen very large. Discontinuities are frequent in computational fluid dynamics problems for instance, where a small change in the parameters results in a change of physics (e.g. laminar to turbulent flow), which creates a discontinuity in the objective.

One remedy to this problem is to add a warping function, either on the output space $y$ \citep{snelson2004warped} or on the input space $\Xset$ \citep{snoek2014input,marmin2018warped}. However, warping usually applies only to continuous functions, and rely on parametric forms, which need to be chosen beforehand and may not adapt to the problem at hand.
A popular alternative is to rely on hierarchical partitions of the input space (assuming stationarity only within each part): see for instance \citet{gramacy2008bayesian,fox2012multiresolution}, but those approaches are in general efficient in small dimension
and with relatively large datasets.

In this work, we propose to apply an ``ordinal'' warping to both input and output data, 
that is, a transformation that only preserves the ordering of the variables. 
A classical (latent) GP model is then fitted to the transformed dataset.
In the output space, this amounts to performing ordinal regression 
using a variational formulation \citep{chu2005gaussian}. 
In the input space, we show that this amounts to defining a large optimisation problem,
which can be solved using standard descent algorithms.

We then study how this model can be used to perform Bayesian optimisation, with minimal use of the original problem metrics. We show that this can be achieved by combining classical acquisition schemes such as upper confidence bound or Thompson sampling and tree search.
Although BO has already been applied to problems with qualitative objectives \citep{gonzalez2017preferential}, we believe that our approach is the first 
that is agnostic to any metric in the input and the output spaces. 

There are a small number of works characterizing the performance of BO on GPs under optimistic acquisition functions.
All these works however consider well behaved GPs where, in particular, the so called information gain is bounded nicely~(see Sec.~\ref{Sec:Analysis} for more detail). In~\cite{srinivas2010gaussian}, an $O(\sqrt {T\gamma})$ upper bound on cumulative regret was shown for GP-UCB a confidence bound based approach where $\gamma$ is an upper bound on information gain. \cite{Chowdhury2017bandit} and~\cite{Javidi} improved the constants in the regret of confidence based policies. Inspired by these works we characterize the regret performance of the proposed confidence bound based policy in the latent space. 

Our approach is illustrated on several toy problems, showing that it is able to optimise severely ill-conditioned and discontinuous functions.

%%%%%%%%%%%%%%%%%%%%%%%%%%%%%%%%%%%%%%%%%%%%%%%%%%%%%%%%%%%%%%%%%%%%%%%%%%%%%%%%%%%%%%%%%
\section{Model}\label{Sec:Model}
\subsection{Definitions and main hypothesis}

\paragraph{Ordinal warping for discrete sets}

We propose to use as a warping any transformation that preserves the ordering of a finite vector with $n$ real values. Without loss of generality, given a set $\{u_1, \ldots, u_n\}$, 
we can write such transformation in the form:

\begin{eqnarray}
 \gamma_n: \Rset &\rightarrow& \Rset \nonumber \\      
u_j &\rightarrow& s_j = \sum_{i=1}^{\rank(u_j)} \delta_i,\label{eq:warp}
\end{eqnarray}
where $\rank()$ denotes the rank function, 
$\rank(u) := \sum_{j=1}^n \texttt{1}_{u_j \leq u}$,
and $\{\delta_1, \ldots, \delta_n\} \in \Rset^{*+}$ are some strictly positive values. 
It is straightforward that $\gamma_n$ is a bijection, moreover $\rank(u_j) = \rank(s_j)$, 
and choosing $\delta_1 = \ldots = \delta_n=1$ results with the rank transformation: $s_j = \rank(u_j)$.

\paragraph{Latent GP model}

Let us assume that we have a set of observations of the form $\{\X_n, \y_n\} = \{\x_1, y_1\}, \ldots, \{\x_n, y_n\}$. 
We define one ordinal warping for $\y_n$, $\gamma_n^y$ (with an underlying set $\delta_1^b, \ldots, \delta_n^b$) and 
$d$ warpings for each dimension of $\X_n$, $\gamma_n^{1}, \ldots, \gamma_n^{d}$ (each dimension $j$ with an underlying set $\delta_1^j, \ldots, \delta_n^j$). 
For each $1 \le i \le n$, we denote:
\begin{eqnarray}
    f_i = f(\s_i) &=& \gamma_n^y(y_i), \\
    \s_i &=& \left[\gamma_n^{1}(x^d_i), \ldots, \gamma_n^{d}(x^d_1) \right] = \gamma_n^\x(\x_i).
\end{eqnarray}

The overall idea is that both $\y_n$ and $\X_n$ can be mapped respectively through $\gamma_n^y$ and $\gamma_n^\x$ to $\f(\mathbf{S}_n) = \{f(\s_1), \ldots, f(\s_n)\}$ and $\mathbf{S}_n = \{\s_1, \ldots, \s_n \}$ such that:
\begin{equation}
    \f(\mathbf{S}_n) \sim \nlaw \left(0, \mathbf{K}_{\s_n \s_n} \right),
\end{equation}
% with $[\mathbf{K}_{\s_n \s_n}]_{ij} = k(\s_i, \s_j)$.
with $\mathbf{K}_{\s_n \s_n} = \left [k \left(\s_i, \s_j \right) \right]_{1 \le i,j \le n}$ and 
$k(\cdot, \cdot)$ a stationary covariance kernel, for instance of the exponential or Mat{\'e}rn classes \citep{williams2006gaussian}.
The input warping is illustrated in Figure \ref{fig:warping_explained}.

\begin{figure}[ht]
 \centering
  \includegraphics[trim=0mm 0mm 0mm 10mm, clip, width=.9\linewidth]{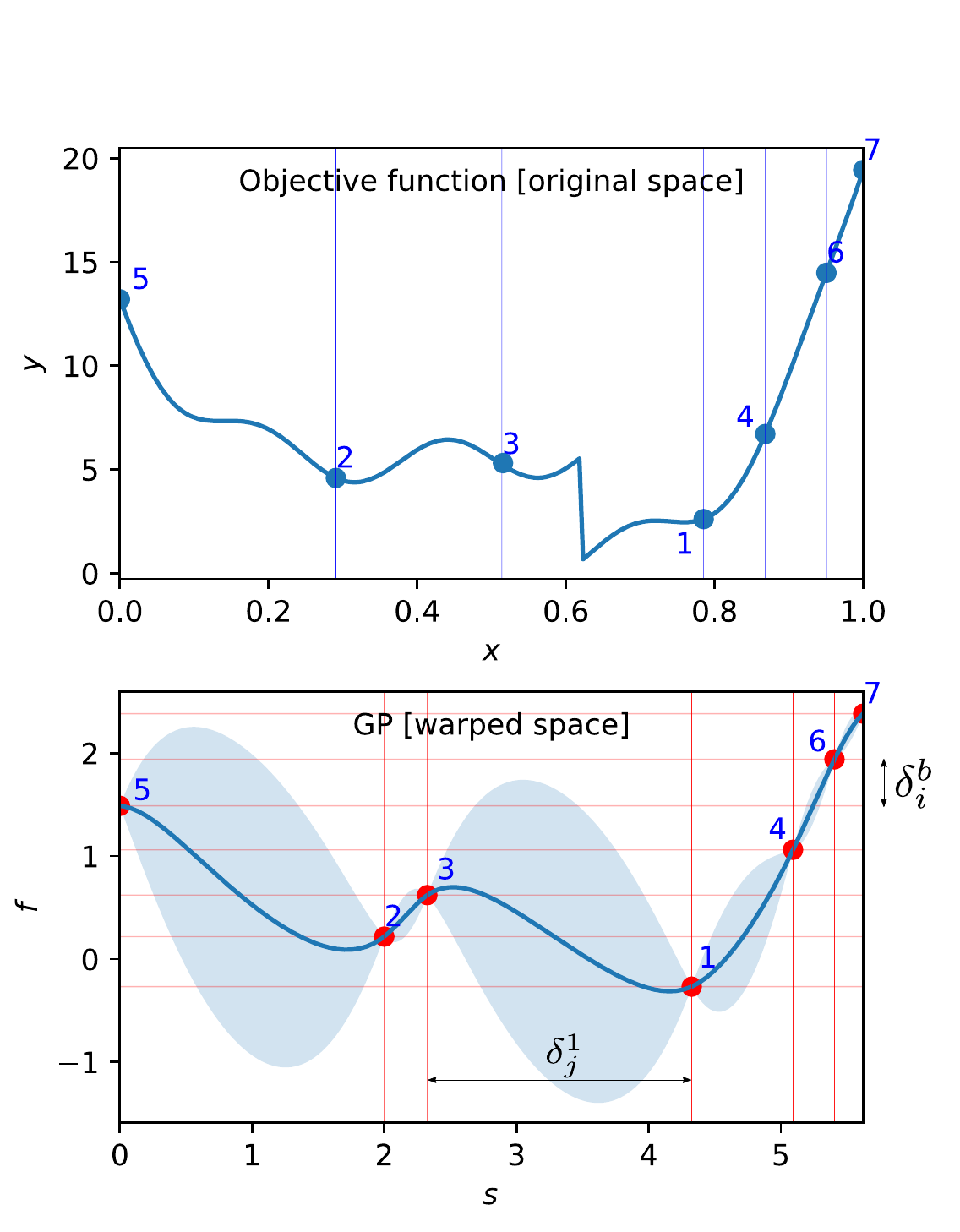}
  \caption{Original and warped spaces. Notive how the ordering is preserved from both $x$ to $s$ and $y$ to $f$.}\label{fig:warping_explained}
\end{figure}

Intuitively, such an approach allows us to tackle the problem where the user only returns pairwise comparisons, such as ``$y_1$ is smaller than $y_2$
and $x_1$ is larger than $x_2$'', which makes it insensitive to scales.
By considering a stationary GP for $f$, we report all the modeling difficulty to the warping step.
Note that although such warpings may be very difficult to infer over continuous spaces, 
we only consider here discrete sets as in Equation \ref{eq:warp}.

\subsection{Learning $f, \gamma_n^\x$ and $\gamma_n^y$ using variational inference}\label{sec:learning}

In the classical GP regression framework, observations $y_i$'s are assumed to correspond 
to evaluations of a latent GP corrupted by Gaussian noise, $y_i = f(\s_i) + \varepsilon_i$.
By doing so, the likelihood function $p(y_i | f_i)$ is set as Gaussian, which allows 
to apply the classical Bayes' rule and obtain a posterior distribution on $\f(\mathbf{S}_n)$:
\begin{equation*}
    p(\f_n|\y_n) = \frac{p(\f(\mathbf{S}_n))\prod^{n}_{i=1}p(y_{i}|f(\s_i))}{\int p(\f(\mathbf{S}_n))\prod^{n}_{i=1}p(y_{i}|f(\s_i))\,d\f(\mathbf{S}_n)}.
\end{equation*}
As all quantities are Gaussian, the posterior distribution can be expressed in closed form.

In the non-conjugate case, exact computation is not tractable and one must resort to approximations.
Variational Inference (VI), which consists in minimising the Kullback-Leibler
divergence between the approximate and the true posterior, has proven to be an effective approach 
in this context. 
We show in the following that applying ordinal warping to the outputs amounts to choosing a non-conjugate likelihood.
We then express the corresponding classical VI problem formulation, which amounts to optimising a lower bound on the marginal log-likelihood. Finally, we show how we can incorporate the parameters of the input warping into the VI problem and learn the warping parameters along with the VI ones.

\paragraph{Output ordinal warping using ordinal regression likelihood}
% Instead of specifying the values $\{\delta_1, \ldots, \delta_n\}$, 
We follow the model of \citet{chu2005gaussian}, that specifies bins for each observation.
Define $b_0 = -\infty$, $b_1$ is an arbitrary real value, $\forall i \in [1, n-2]$, $b_{i+1} = b_i + \delta^b_i$, and $b_{n} = +\infty$.
Assuming that the observations are in increasing order ($f_1 \leq f_2 \leq \ldots$),
the likelihood functions can be expressed as:
\begin{equation}
 p(y_i | f(\s_i)) = \Phi\left(\frac{b_i - y_i}{\sigma} \right) - \Phi\left(\frac{b_{i-1} - y_i}{\sigma}\right),
 \label{eq:ordinal-likelihood}
\end{equation}
with $\Phi$ the standard Gaussian cumulative distribution function.
The term $\sigma$ corresponds to a (small) noise in the latent functions.

\paragraph{ELBO}
We now follow the classical Variational GP (VGP) framework \citep{titsias2009,hensman2013}.
To compute the data likelihood, we only need the marginal posterior distribution of the GP at the (warped) 
% First, we introduce a set of \emph{inducing variables}, which main use is to specify the function value of the posterior GP at 
inputs $\{\s_1, \ldots, \s_n\}$, denoted as $\f_n = \{f(\s_1), \ldots, f(\s_n)\}$. 
We propose an approximate posterior in which we directly parametrise the distribution of function values at the inputs as a multivariate normal, $q(\f_n) = \nlaw(\vm, \MS)$
with mean $\vm \in \Rset^n$ and covariance $\MS \in \Rset^{n\times n}$,
where $\vm$ and $\MS$ are optimisation parameters. 

Conditioned on $q(\f_n)$ we obtain the approximate posterior GP where the mean $\mu_n(\cdot)$ and the covariance $k_n(\cdot, \cdot)$ can be calculated in closed form:
\begin{eqnarray*}
\label{eq:qf}
  \mu_n(\cdot) &=& \mathbf{k}_{\s_n}^\top(\cdot) \mathbf{K}_{\s_n \s_n}^{-1} \vm \quad\text{and} \\
  k_n(\cdot, \cdot) &=& k(\cdot, \cdot) + \\
  &&\mathbf{k}_{\s_n}^\top(\cdot) \mathbf{K}_{\s_n \s_n}^{-1}(\MS - \mathbf{K}_{\s_n \s_n}) \mathbf{K}_{\s_n \s_n}^{-1}\mathbf{k}_{\s_n}(\cdot),
\end{eqnarray*}
where $\mathbf{k}_{\s_n}(\cdot) := \left[k(\s_i, \cdot)\right]_{i=1}^n \in \Rset^n$.
% and $[\mathbf{K}_{\s_n \s_n}]_{ij} = k(\s_i, \s_j)$.

With this approximation in place we can set up our model's optimisation objective, which is a lower bound on the log marginal likelihood \citep[ELBO,][]{hoffman2013}, equal to
\begin{equation*}
       \mathcal{L} = \sum_{i=1}^n \esp_{q(\f(\mathbf{S}_n))}
        \big[ \log p(y_i | f_i) \big]
        - \KL\big[q(\f(\mathbf{S}_n)) || p(\f(\mathbf{S}_n))\big], \label{eq:elbo}
\end{equation*}
where $p(\f(\mathbf{S}_n)) = \nlaw(0, \mathbf{K}_{\s_n \s_n})$ and $p(y_i | f(\s_i))$ is the ordinal likelihood \eqref{eq:ordinal-likelihood}.

In practice, the expectation $\esp_{q(\f_n)}$ cannot be evaluated analytically, but as the likelihood factorises over data points this is just a one-dimensional integral, which can easily be computed numerically using Gauss--Hermite quadrature.

\paragraph{Optimisation} 
Now, $\mathcal{L}$ is optimised with respect to three sets of variables:
a) the inducing variables, $\vm \in \Rset^n$ and $\MS \in \Rset^{n \times n}$;
b) the $n-1$ likelihood parameters, $\delta^b_2, \ldots, \delta^b_{n-1} \in \Rset^+$ and $\sigma \in \Rset^+$;
c) the $d \times (n-1)$ input warping parameters, as each warping is defined using $n$ $\delta \in \Rset^{*+}$ values, but we can set arbitrarily $\delta_0^1 = \ldots = \delta_0^d = 0$.

Note that since the distances between points are set by the $\delta_i^j$'s and
the amplitude of the response is set through the $b_i$'s, we can define $\MS$ using a stationary kernel with unit variance and lengthscale.

To ease the problem resolution in practice, we restrict $\MS$ to be diagonal and add a set of boundary constraints when solving the ELBO optimisation problem. Each value $\delta^b_i$ and $\delta_i^j$ are bounded between a (small) strictly positive value and a maximum. As we use a unit kernel variance and lengthscale, we can set those maxima to values related to resp. the amplitude of the GP and such that the covariance between two points distanced by $\delta_{max}$ is close to zero). Note that during BO, these bounds can be reduced to bound the maximum variation of $s_j$ between two consecutive steps, as we detail in Section 4. In our implementation, this problem is solved by stochastic gradient descent, leveraging automatic differentiation tools. The bound constraints are handled using logistic transformations.

%%%%%%%%%%%%%%%%%%%%%%%%%%%%%%%%%%%%%%%%%%%%%%%%%%%%%%%%%%%%%%%%%%%%%%%%%%%%%%%%%%%%%%%%%%%%%%%%%%%%%%%%%%%%
\section{Bayesian optimisation}

\subsection{Acquisitions on latent and original spaces}
Standard BO algorithms work as follow. An initial set of experiments $\{\X_{n_0}, \y_{n_0}\}$
is generated, typically using a space-filling design \citep{pronzato2012design} over $\Xset$,
and a GP model is trained on this dataset. Then, an \textit{acquisition rule} is applied repeatedly,
that consists of evaluating $y$ at the input $\x$ that maximises an \textit{acquisition function}.
Every time a new data point is acquired, the GP posterior distribution is updated to account for it.

The acquisition function is based on the GP distribution and balances between exploration (high GP variance) and exploitation (low GP mean). Typical acquisitions include 
\textit{Expected improvement} \citep[EI,][]{jones1998efficient} and \textit{upper confidence bound} \citep[UCB,][]{srinivas2010gaussian}.

In our case, given that $f$ is a stationary GP, it is direct to predict the posterior distribution $f(\s_{\text{new}}) \sim \nlaw(\mu_n(\s_{\text{new}}), k_n(\s_{\text{new}}, \s_{\text{new}}))$ for a new value $\s_{\text{new}}$. Hence, classical acquisition functions apply and it is straightforward to select a point $\s_{\text{new}}$ to acquire.

However, the mapping $\gamma_n^\x$ is only defined from $\X_n$ to $\mathbf{S}_n$,
and it is not possible to find the $\x_{\text{new}}$ that corresponds to $\s_{\text{new}}$
without 
either 1- creating a new ordinal warping $\gamma\+$, which requires the value of $f(\x_{\text{new}})$ 
or 2- generalising the warpings, say by linear interpolation\footnote{e.g. if $\s_{\text{new}}=\frac{\s_1 + \s_2}{2}$ then we choose $\x_{\text{new}}=\frac{\x_1 + \x_2}{2}$.}, which contradicts our metric-free principle.
% it is not possible to predict the value at $\x_{\text{new}} \in \Xset$ 

Instead, we leverage the fact that the ordinal input warping implies a one-to-one mapping between hyper-rectangle cells,
i.e.  $\s_{\text{new}} \in \Omega_{\s_{\text{new}}} \subset \Sset \Leftrightarrow \x_{\text{new}} \in \Omega_{\x_{\text{new}}} \subset \Xset$, which are determined 
by the rank of $\x_{\text{new}}$ with respect to the existing $\x_1, \ldots, \x_n$ (see Figure \ref{fig:ts_lcb_explained}). 

Hence, by choosing 
$\x_{\text{new}}$ we guarantee to have $\s_{\text{new}}$ within given bounds, 
but by being ``truly agnostic'' with respect to the metrics of the original space, 
we cannot be more precise about the location of $\s_{\text{new}}$.

Then, instead of using acquisitions functions that return a value for a new point, 
we need functions that evaluate cells.
We show in the following how to adapt two acquisition strategies, LCB and Thompson sampling, 
to this framework.

\subsection{Lower confidence bound}
The GP-UCB strategy of \citet{srinivas2010gaussian} uses an optimistic upper confidence bound for $f$ aiming at a maximization problem. Similarly, we use a lower confidence bound as follows for our minimization problem:
\begin{equation*}
 lcb(\s) = \mu_n(\s) - \beta_n \sqrt{k_n(\s, \s)},
\end{equation*}
where $\beta_n \in \Rset^+$ is a quantity that generally grows with $n$. 
Assume now that $\x \in \Omega_x \subset \Xset$.
The LCB (somehow twice optimistic) would become:
\begin{equation}\label{eq:LCB}
 LCB(\Omega_x) = \min_{s \in \Omega_s} lcb(\s).
\end{equation}

\subsection{Thompson sampling (TS)}
The principle of TS is to choose actions in proportion to the probability that they are optimal.
For GPs, a simple way to do so is to generate a sample from the posterior of $f$ and pick its minimiser as the next point to evaluate. This requires however to discretise the input space (using a fine Cartesian grid or a low discrepancy sequence).

The main difficulty of the vanilla TS-GP is the discretisation of the input space. 
This problem is removed here as the action to take is to \emph{choose the best cell}.
We may choose a cell according to the probability that it maximises the expected reward.
This can be achieved by repeatedly 1- sampling one random $\s_i$ in each cell $i$, 2- sampling jointly $f(\s_1), \ldots, f(\s_m)$, and 3- recording which cell is optimal. Then, the cell where a new sample is drawn is chosen with probability proportional to the number of times it was optimal.

\subsection{Domain decomposition}
With $n$ points, the ordinal warping naturally defines a decomposition of the search space into $(n-1)^d$ cells.
Our first strategy is to consider the acquisition values over all of those cells. We call ``exhaustive'' such approach.
Although this works without issue in small dimension ($d\leq2$) and low data regime (say, $n \leq 100$), 
the number of cells grows very quickly with the number of added observations and dimension, 
and computing the acquisition value of each cell can rapidly become impractical.

An alternative is to use a hierarchical partitioning, that is, starting from the initial exhaustive decomposition
induced by the first $n_\text{init}$ points, 
to only create new cells by dividing the cell where the new observation is generated.
With this strategy, only $2^d - 1$ cells are added for every new observation, leading to a total of 
$(n_\text{init} - 1)^d + (n - n_\text{init}) (2 ^d - 1)$.
We refer to this approach as ``tree search''.
Both approaches are illustrated in Figure \ref{fig:cells}.

\begin{figure}[ht]
 \centering
  \includegraphics[trim=10mm 10mm 10mm 10mm, clip, width=.9\linewidth]{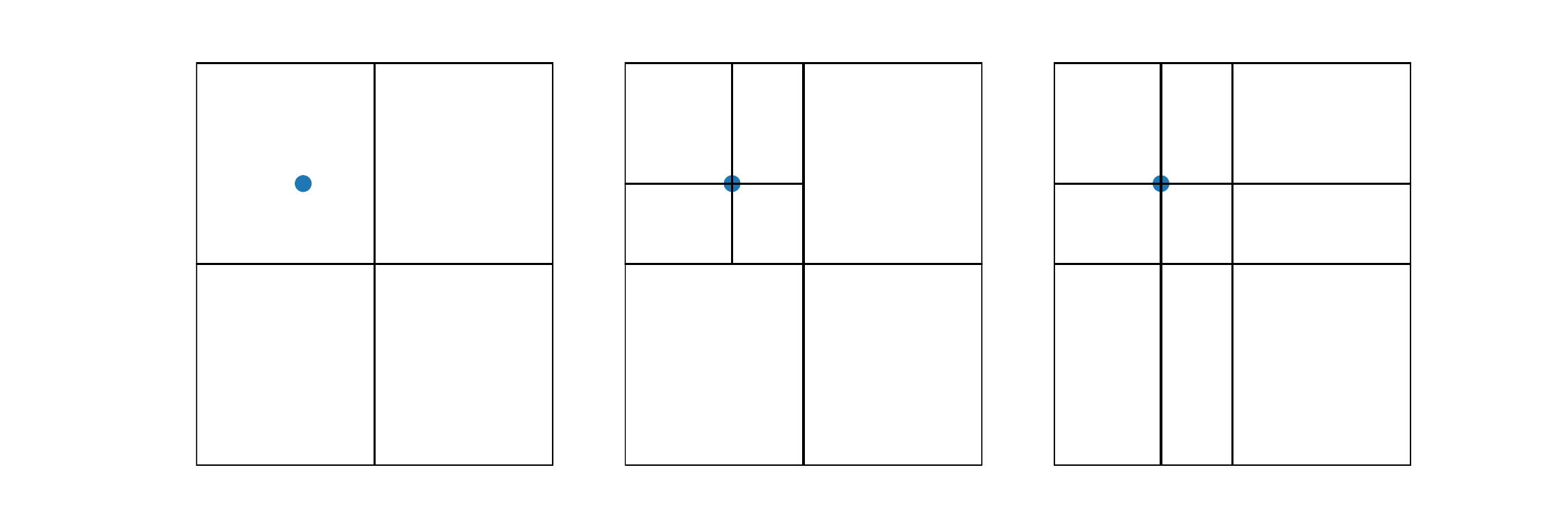}
  \caption{Domain splitting when a new observation $\s_{new}$ (blue dot) is added to an initial four-cell partition: 
  tree-search (middle, leading to seven cells), exhaustive (right, nine cells).}\label{fig:cells}
\end{figure}

\subsection{Algorithms}
The pseudo-code of the Thompson sampling algorithm is given in \cref{alg:ts}. % and \cref{alg:lcb}.
The LCB algorithm has a simpler but similar structure, as the steps 9:13 are replaced by computing the LCB criterion of \cref{eq:LCB}, and the cell is simply chosen as the one that maximises the LCB.
The acquisition steps are illustrated in Figure \ref{fig:ts_lcb_explained}.
On both cases, the warping is updated every time a new input point is observed to account for the addition of a new pair $\{\s_{n+1}, f_{n+1}\}$,
which is done by introducing a new $\delta_b$ value and a set of $d$ $\delta_i^j$ values.
The variational parameters $\vm$ and $\MS$ are augmented, resp. with one value and with one row and column.
Then, the ELBO is trained again, which updates all the parameters listed in Section \ref{sec:learning}.

\begin{figure*}[ht]
 \centering
  \includegraphics[trim=29mm 0mm 25mm 4mm, clip, width=\linewidth]{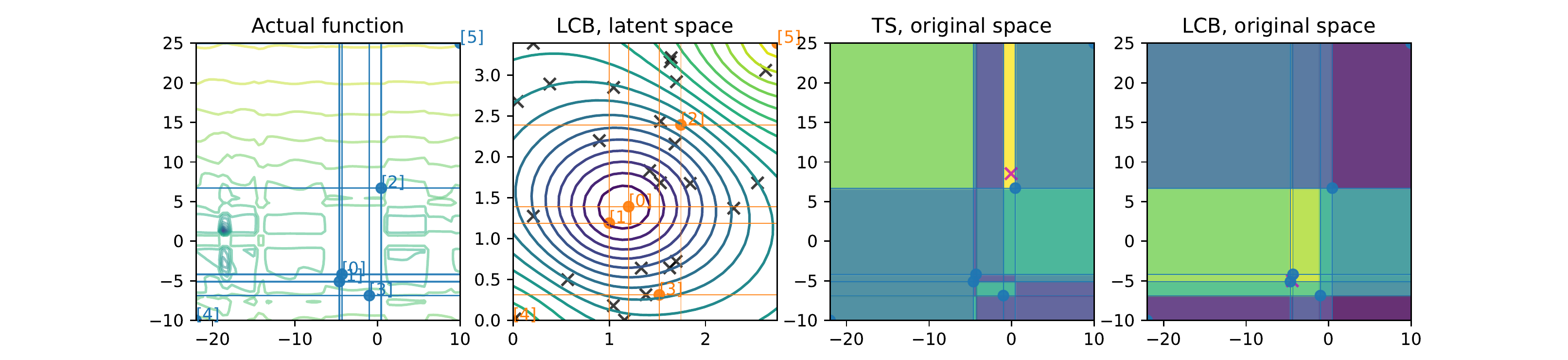}
  \caption{From left to right: 
  A) Contour lines of the objective and initial $\x_n$ values (blue dots). 
  The numbers correspond to the rank of the observations.
  B) Contour lines of the lcb in the latent space, along with the $s_n$ values (orange dots).
  Sampled minima for TS are shown in black crosses.
  C) Probability of containing the minimum by cell in the original space
  D) LCB by cell in the original space.
  New proposed samples are shown in purple crosses.}\label{fig:ts_lcb_explained}
\end{figure*}

\begin{algorithm}[ht]
    \centering
    \caption{Pseudo-code for the Thompson sampling with tree search}\label{alg:ts}
    \begin{algorithmic}[1]
	\State Choose $n_\text{init}$, $n_\text{run}$
    \State Sample $X_{n_\text{init}-2}$ uniformly on $\Xset$. 
    \State Add extremes: $X_{n} =\{X_{n_\text{init}-2}, x_{\min}, x_{\max} \}$
    \State Evaluate $y(X_{n})$.
    \State Create $(n-1)^2$ cells using $\X_{n}$.
    \State Initialize warping and variational parameters
    \State Create and train GP model for $f$ by optimizing the ELBO
    \For{$i \gets n_\text{init} + 1$ to $n_{run}$}
      \State Update cells of $\Sset$ according to $\s_n$ values.
      \For{$k \gets 1$ to $M$}
        \State Generate $\{s^{TS}_1, \ldots, s_{n_\text{cells}}^{TS}\}$ sampling points in $\Sset$ (each $s^{TS}_i$ randomly drawn inside a different cell).
        \State Draw one sample of $f(\{s^{TS}_1, \ldots, s_{n_\text{cells}}^{TS}\})$
        \State Record which cell contains the minimiser of the sample
      \EndFor
      \State Choose one cell in $\Xset$ with probability according to the number of times it contained the minimiser.
      \State Generate a random $x_\text{new}$ value inside this cell
      \State Evaluate $y(x_\text{new})$
      \State Split the cell into 4 new ones.
     \State Update warping and variational parameters
     \State Optimize the ELBO
     \EndFor
    \end{algorithmic}
\end{algorithm}

\section{Analysis}\label{Sec:Analysis}
Our method is designed agnostic to the metric in the original space. Thus, regret in original space is not well-defined in the scope of the (very general) formulation of this paper. Obtaining results in the original space is not out of reach, but implies making substantially restricting  (e.g., Lipschitz) assumptions on the original function, which in some sense defeats the purpose of this work. Hence, we focus on the latent space to show the convergence of the method.

Our analysis is inspired by the analysis of GP-UCB in~\cite{srinivas2010gaussian}. In~\cite{srinivas2010gaussian}, the observation locations are static, while in our case they vary over time. Our problem thus has an additional difficulty which requires new developments. Specifically, we use a bound on the amount of variation in the location of observations to establish new bounds on information gain that results in regret bounds for a lcb method with particular dynamics of observation points. This may be a valuable contribution for other contexts with dynamic data sets.

The values of observation points ($\s_j$) vary in each iteration with injecting new observations as a result of the update of warping parameters described in Sec.~\ref{Sec:Model}. Specifically, thus far we have simplified the notation $\s_j^{(n)}$ of the value of observation points in the latent space by removing the superscripts corresponding to the iterations. The superscript specifies the number of observations used in determining the warping parameters. Thus, $s^{n}_j$ denotes the value of $j$-th observation when $n$ observations are used in determining the warping parameters. Notice that $s_n^{n-1}$ denotes the location of $n$-th observation based on the warping determined by the $n-1$ previous observations, while $s_n^n$ denotes its location after updating the warping with this nth observation.

The performance measure specified as regret defined as the cumulative loss in $f$ compared to its optimum value. Let $\s^*\in\text{argmin}_{\Sset}f(\s)$. Define 
\begin{eqnarray}
R(N,f,lcb) = \sum_{n=1}^N f(\s_{n}^{(n-1)}) - f(\s^*),
\end{eqnarray}
where $\s_{n}^{(n-1)}$ is the new observation point at iteration $n$ conditioned on the previous observation points $\{\s_j^{(n-1)}\}_{j=1}^{n-1}$. The regret order determines the rate of convergence to optimum value of $f$. Not only a sublinear regret guarantees the convergence to the optimum value, the regret measure also accounts for the intermediate values of the observations and makes sure the overall loss is not too large.  

Regret can be bounded in terms of the maximum amount any algorithm could learn about the objective function. \citet{srinivas2010gaussian} excellently characterized this intuition using an information theoretic measure referred to as information gain $I$ whose value depends on the observation points and the kernel function. Specifically,
\begin{eqnarray}\label{eq:Infogain}
I_n(\{\s_{j}\}_{j=1}^n) = \frac{1}{2}\log |\mathbb{I}_n+\sigma^{-2}\K_{n}|,
\end{eqnarray}
where $\K_n =[\mathbf{k}_n(\s_i,\s_j)]_{i,j\in[n]}$ and $\mathbb{I}_n$ is the identity matrix. 
The regret upper bound is established based on the following two lemmas. In Lemma~\ref{Lemma:InstReg}, instantaneous regret at iteration $n$ is upper bounded by the variance $\sigma^2_{n-1}(\s_n^{(n-1)})$ of the new observation point at iteration $n$ up to constants. In Lemma~\ref{Lemma:VnUB}, the cumulative value of such variances defined as 
\begin{eqnarray}\label{eq:Vn}
V_N = \sum_{n=1}^N \sigma^2_{n-1}(\mathbf{s}^{(n-1)}_n).
\end{eqnarray}
is upper bounded by the upper bound on information gain from $N$ observations up to a constant independent of $N$. Combining these two lemmas, Theorem~\ref{The:RegUB} gives the upper bound on the regret of lcb policy. 

Our analysis requires that the observation set $\Sset$ is compact; in particular, $\Sset$ is a bounded subset of $\R^d$. Without loss of generality we assume
\begin{eqnarray}\label{eq:compact}
\Sset=[0,1]^d. 
\end{eqnarray}
We further have the following assumption on the smoothness of the kernel (the same as in~\cite{srinivas2010gaussian}).  
\begin{assumption}\label{Ass:smoothness}
For some constants $L,c_1,c_2>0$,
\begin{eqnarray}
\Pr\left[\sup_{\s\in D}|\frac{\partial g}{\partial s_j}|>L\right]\le c_1\exp\left(-\frac{L^2}{c_2^2}\right).
\end{eqnarray}
\end{assumption}

\begin{lemma}\label{Lemma:InstReg}
For some $\delta\in(0,1)$, let
$\beta_n = 2\log(\frac{2\pi^2n^2}{3\delta})+4d\log(dc_2n\sqrt{\log\frac{2dc_1}{\delta}})$. We have, for all $n\ge1$,
\begin{eqnarray}
f(\s_n^{(n-1)})-f(\s^*) \le 2\beta_n^{\frac{1}{2}}\sigma_{n-1}(\s_n^{(n-1)})+\frac{1}{n^2},
\end{eqnarray}
with probability at least $1-\frac{1}{\delta}$.
\end{lemma}

\emph{Proof.} See Appendix~A.1.

The cumulative variance $V_N$ of the observed points in all iterations is an indicator of the total reduced uncertainty in the value of the function after $N$ observations and is a key parameter in characterizing the regret of lcb. 
Variation of observations points $\s_j^{(n)}$ over iterations however makes it difficult to bound $V_N$. To account for this variation, the following condition is imposed through the warping step 
\begin{eqnarray}\label{eq:vary}
||\s_j^{(n-1)}-\s_j^{(n)}||_2 \le \Delta_n,
\end{eqnarray}
for $1<n\le N$ where $\Delta_n\le \frac{Cd}{n}$ for some constant $C$ independent of $n$ and $d$. 
In practice, Eq. \ref{eq:vary} is enforced by constraining each component of $s_j^{n}$ to be close to those of $s_j^{n-1}$ when re-optimising the ELBO.
% This assumption is reasonable given the bounded support of the space of observations $\Sset=[0,1]^d$. 

\begin{lemma}\label{Lemma:VnUB}
The cumulative variance $V_N$ defined in~\eqref{eq:Vn} is upper bounded as follows
\begin{eqnarray}\label{eq:VnUB}
V_N \le C_1 \gamma_N + C_2\log(N+1),
\end{eqnarray}
where $C_1$ and $C_2$ are constants independent of $N$ (given in appendix A) and $\gamma_n$ is an upper bound on information gain $I_N$ as defined in~\eqref{eq:Infogain}.
\end{lemma}

\emph{Proof.} See Appendix~A.2.

See~\cite{srinivas2010gaussian} for the detail on the value of $\gamma_N$ for several kernels. For example they show a $\gamma_N=O(d\log N)$ and a $\gamma_N=(O(\log N)^{d+1})$ for finite spectrum and Squared Exponential kernels, respectively.

The following regret upper bound follows from the results established in Lemmas~\ref{Lemma:InstReg} and~\ref{Lemma:VnUB}.

\begin{theorem}\label{The:RegUB}
The regret of lcb over the compact set specified in~\eqref{eq:compact} under Assumption~\ref{Ass:smoothness} satisfies,
\begin{eqnarray}\nn
R(N,f,lcb) \le \sqrt{4N\beta_N(C_1\gamma_N+C_2\log(N+1))} +C_3,
\end{eqnarray}
where $C_1$, $C_2$ are the constants in Lemma~\ref{Lemma:VnUB} and $C_3$ is specified in the proof.
\end{theorem}

\begin{proof}[Proof of Theorem~\ref{The:RegUB}]
We can write the cumulative regret as the sum of the instantaneous regrets and use Lemmas~\ref{Lemma:InstReg} and~\ref{Lemma:VnUB} to obtain
{\small{\begin{eqnarray}\nn
R(N,f,lcb) 
&=&\sum_{n=1}^Nf(\s_n^{(n-1)})-f(\s^*)\\\label{st1}
&\le& \sum_{n=1}^N\left(2 \beta_n^{\frac{1}{2}}\sigma_{n-1}(\s_n^{(n-1)}) +\frac{1}{n^2}\right)\\\label{st2}
&\le& \sqrt{4N\sum_{n=1}^N \beta_n\sigma^2_{n-1}(\s_n^{(n-1)})}+\sum_{n=1}^{N}\frac{1}{n^2}\\\label{st3}
&\le& \sqrt{4N\beta_N\sum_{n=1}^N \sigma^2_{n-1}(\s_n^{(n-1)})} +\frac{\pi^2}{6}~~~~~\\\nn
&\le& \sqrt{4N\beta_NV_N} +\frac{\pi^2}{6}\\\label{st4}
&\le& \sqrt{4N\beta_N(C_1 \gamma_N + C_2\log(N+1))} +\frac{\pi^2}{6}.~~~~~~~
\end{eqnarray}}}
Inequality~\eqref{st1} comes from Lemma~\ref{Lemma:InstReg},~\eqref{st2} is a result of Cauchy-Schwarz inequality,~\eqref{st3} is obtained by the fact that $\beta_n$ is increasing in $n$ and~\eqref{st4} is a direct application of Lemma~\ref{Lemma:VnUB}. The theorem holds with $C_3=\frac{\pi^2}{6}$.

\end{proof}

\section{Experiments}
As a proof of concept, we consider a set of toy problems: a 1D function (depicted in Figure \ref{fig:warping_explained}), three 2D functions, 
one with many discontinuities and two from a classical optimisation benchmarks \citep{hansen2016coco}, namely "bent cigar" and "different power" 
(depicted in Figure \ref{fig:multistep}) and a classical 4D function, "Hartman" \citep{dixon1978global}.
The 1D function has a critical discontinuity at optimum, the "bent cigar" and "different power" functions 
are unimodal but very challenging, due to a high conditioning and a very narrow optimum region. 
The other 2D function is multimodal, contains many discontinuities in the optimal region and has high conditioning.
The 4D function serves as a reference, as vanilla BO is known to perform well on it.

We compare our two algorithms, TS and UCB to a vanilla BO based on expected improvement. 
For both TS and UCB, we use the tree-search partition; UCB is run with a fixed $\beta_t$ of $3.0$.
For all methods, an initial set of five experiments is generated by space-filling design,
followed by 20 iterations.
Each strategy is run 10 times with different initial conditions. 
Performance is measured in terms of cumulative regret and reported in Figure \ref{fig:regret}.

All methods are implemented using gpflow \citep{matthews2017gpflow}, and all GPs 
use the same Matern3/2 kernel. For the ordinal approach,
the optimisation problem \ref{eq:elbo} is solved using Adam \citep{kingma2014adam}.

We can see that on the challenging functions, while vanilla BO struggles to optimise the function, 
both of our algorithms significantly outperform BO, in particular during the first steps, 
despite the difficulty of the problem. Here, LCB performed better than TS.
On the classical Hartman function, our approach is only marginally outperformed by standard BO on the classical problem.

Figure \ref{fig:1dexample} shows the state of single run of our approach with LCB (after 8 acquisition points added to an initial set of 5 points).
We see that the latent model accounts for the discontinuity by setting a large distance between the points just before and just after it. As a result, our algorithm is capable of exploring the discontinuity region, while a vanilla BO algorithm would have rejected it rapidly.
One may also notice that a large region (left) is considerably reduced in the latent space. Although this appears as beneficial on this run,
it also indicates that our algorithm might be less global than vanilla BO. 
Such issue could be addressed by mixing our approach with metric-based strategies, for instance by sampling from time to time inside the largest cell in the original space.

\begin{figure}[ht]
 \centering
   \includegraphics[trim=40mm 10mm 270mm 15mm, clip, width=.9\linewidth]{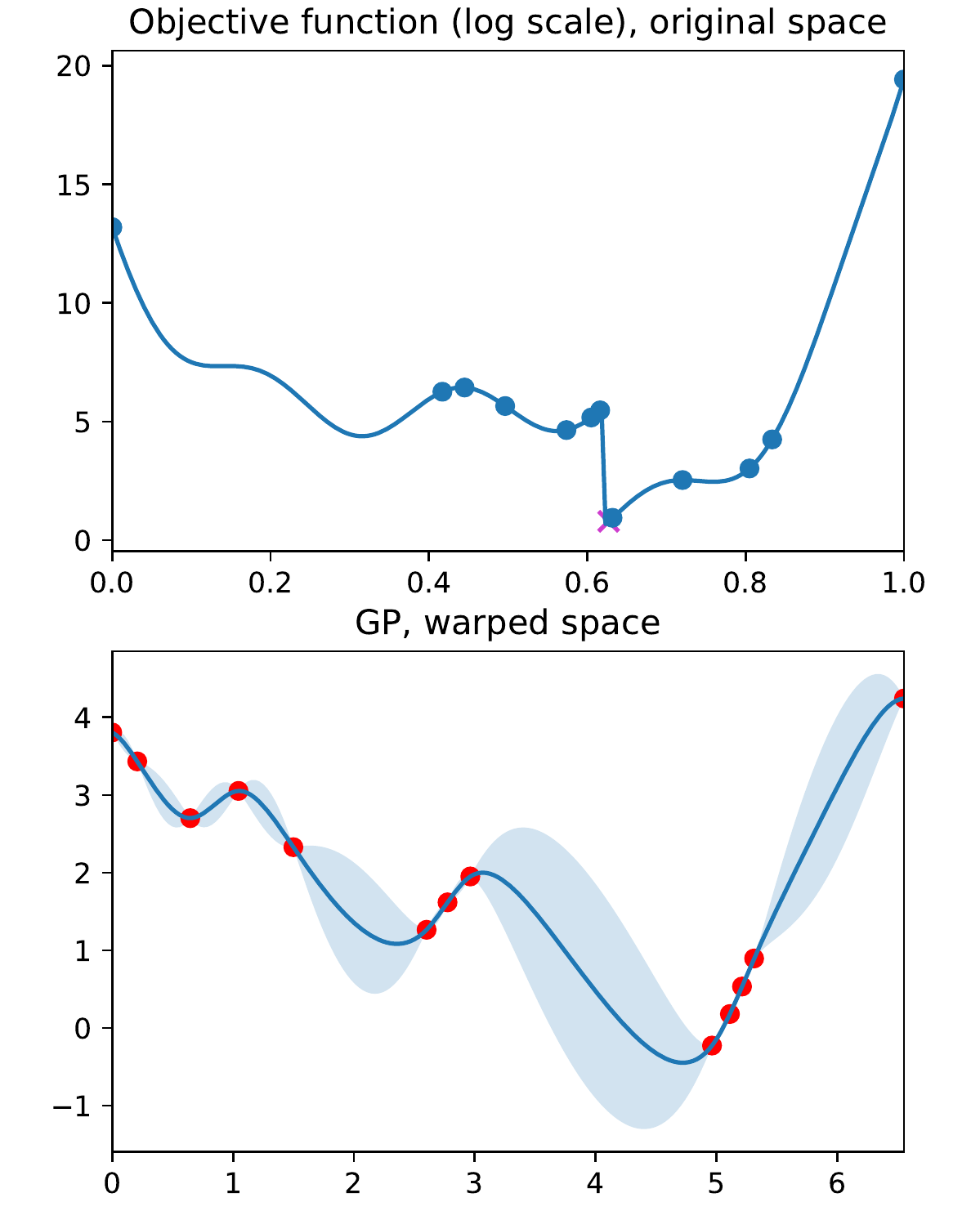}
   \caption{Example run of LCB on the 1D function.}\label{fig:1dexample}
 \end{figure}

Figure \ref{fig:2dexample} shows a single run of our approach with LCB (with 30 acquisition points).
One can observe that most observations form clusters around local and global optima,
while some regions are largely ignored.

\begin{figure}[ht]
 \centering
  \includegraphics[trim=0mm 0mm 0mm 9mm, clip, width=\linewidth]{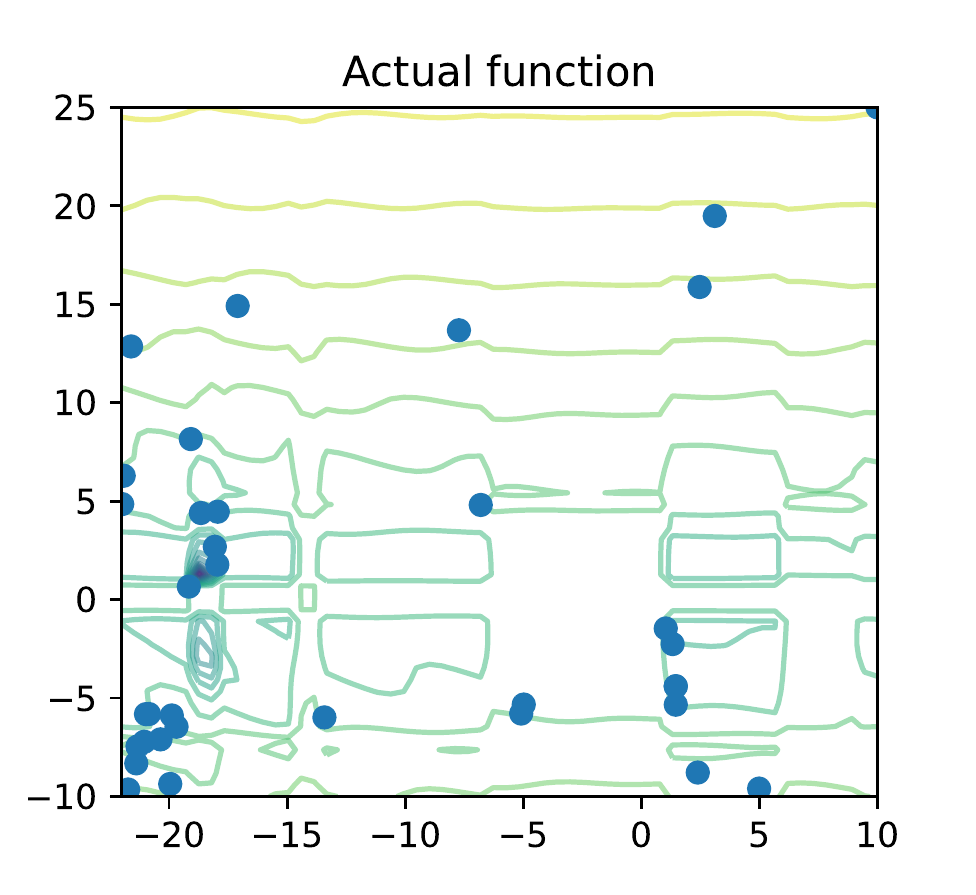}
  \caption{Example run of LCB on the "many steps" function.}\label{fig:2dexample}
\end{figure}

\section{Concluding comments}
In this work, we proposed a novel BO approach that does not consider the values of the inputs or outputs, but only their respective ordering. Our algorithm is based on a Variational GP model and a set of ordinal warpings. We showed how such model could be used to refine either strategies of LCB or Thompson sampling. We proved an upper bound on the regret of confidence bound based approached in the warped space, and demonstrated the capability of our algorithm on a challenging toy problem. 

Future work may include the analysis of Thompson sampling and a more comprehensive experimental comparison to the existing approaches over a wider range of ill-conditioned objective functions. 

\section*{Appendix A.1: Proof of Lemma~1}

From assumption~1, for all $\s\in\Sset$, with probability at least $1-dc_1\exp\left(-\frac{L^2}{c_2^2}\right)$,
\begin{eqnarray}
|f(\s)-f(\s')|\le L||\s-\s'||_1.
\end{eqnarray}
Let $\delta = 2dc_1\exp\left(-\frac{L^2}{c_2^2}\right)$. We have that with probability greater than $1-\frac{\delta}{2}$,
\begin{eqnarray}\label{sta1}
|f(\s)-f(\s')|\le c_2\sqrt{\log\frac{2dc_1}{\delta}}||\s-\s'||_1.
\end{eqnarray}

Assumption $\Sset\subseteq[0,1]^d$ indicates that $0\le s_j\le 1$ for all $\s\in \Sset$ and $1\le j\le d$. Let $\Sset_n$ be a discretisation of $\Sset$ with size $|\Sset_n|=K_n^d$ (with $K_n\in\Nset$) at iteration $n$ such that, for all $\s\in \mathcal{S}$ 
\begin{eqnarray}\label{eq:closeness}
||\s-c_{\Sset_n}(\s)||\le \frac{ud}{K_n},
\end{eqnarray}
where $c_{\Sset_n}(\s) = \min_{\s'\in\Sset_n}||\s'-\s||$ is the closest point to $\s$ in $\Sset_n$. 
This discretisation is possible by uniformly spreading $K$ points along each coordinate of $\Sset$.  

Choosing $K_n=dn^2c_2\sqrt{\log\frac{2dc_1}{\delta}}$, from~\eqref{sta1} and~\eqref{eq:closeness}, we have with probability at least $1-\frac{\delta}{2}$
\begin{eqnarray}\label{eq24}
|f(\s)-f(c_{S_n}(\s))|&\le& c_2\sqrt{\log\frac{2dc_1}{\delta}} \frac{d}{K_n}
=\frac{1}{n^2}.~~~~
\end{eqnarray}

$\frac{f(\s)-\mu_{n-1}(\s)}{\sigma_{n-1}(\s)}$ has a normal distribution. Thus
\begin{eqnarray}
\Pr[|f(\s) - \mu_{n-1}(\s)|> \beta_n^{\frac{1}{2}}\sigma_{n-1}(\s)]\le \exp(-\frac{\beta_n}{2}).
\end{eqnarray}

Replacing $\beta_n$ and using union bound over all $n$ and $\s\in\Sset_n$, we have with probability at least $1-\frac{\delta}{2}$, for all $\s\in\Sset_n$ and $n\ge1$
\begin{eqnarray}\label{eqbou1}
|f(\s) - \mu_{n-1}(\s)|\le \beta_n^{\frac{1}{2}}\sigma_{n-1}(\s).
\end{eqnarray}

Now we have all the material needed to prove the lemma. By definition of the acquisition rule:
\begin{eqnarray*}\nn
&&\hspace{-3em}\mu_{n-1}(\s^{(n-1)}_{n})-\beta_n^{\frac{1}{2}}\sigma_{n-1}(\s^{(n-1)}_n) \le\\
&&\mu_{n-1}(c_{{n-1}}(\s^*))-\beta_n^{\frac{1}{2}}\sigma_{n-1}(c_{{n-1}}(\s^*)).~~~~ 
\end{eqnarray*}

Using an union bound on~\eqref{eq24} and~\eqref{eqbou1}, we have with probability at least $1-\delta$
\begin{eqnarray*}
f(\s^*) \ge \mu_{n-1}(c_{n-1}(\s^*))  -\beta_n^{\frac{1}{2}}\sigma_{n-1}(c_{{n-1}}(\s^*))-\frac{1}{n^2},
\end{eqnarray*}
which shows with probability at least $1-\delta$
\begin{eqnarray}\label{la0}
f(\s^*) \ge \mu_{n-1}(\s^{(n-1)}_{n})-\beta_n^{\frac{1}{2}}\sigma_{n-1}(\s^{(n-1)}_n) -\frac{1}{n^2}.
\end{eqnarray}

Applying~\eqref{eqbou1} to $\s_{n}^{(n-1)}$ we have
\begin{eqnarray}\nn
&&\hspace{-1em}f(\s_{n}^{(n-1)}) - f(\s^*)\\\nn
&\le& \mu_{n-1}(\s_{n}^{(n-1)}) +\beta_n^{\frac{1}{2}}\sigma_{n-1}(\s_{n}^{(n-1)})-f(\s^*)\\\label{la1}
&\le&2\beta_n^{\frac{1}{2}}\sigma_{n-1}(\s_{n}^{(n-1)}) +\frac{1}{n^2},
\end{eqnarray}
where~\eqref{la1} is a result of~\eqref{la0}.

\section*{Appendix~A.2: Proof of Lemma~2}

The following equation is a direct result of Lemma 5.4 in~\cite{srinivas2010gaussian} on~$I_n$:
\begin{eqnarray}\nn
&&\hspace{-4em}I_n(\{s^{(n-1)}_j\}_{j=1}^n) - I_{n-1}(\{s^{(n-1)}_j\}_{j=1}^{n-1}) \\\nn &=&\frac{1}{2}\log(1+\frac{\sigma_{n-1}^2(s_n^{(n-1)})}{\sigma^2}).
\end{eqnarray}
 
By definition of $V_n$, we have
{\small{\begin{eqnarray}\nn
V_n - V_{n-1} 
&=& \sigma_{n-1}^2(s_n^{(n-1)}) 
= \sigma^2\frac{\sigma_{n-1}^2(s_n^{(n-1)})}{\sigma^2}
\end{eqnarray}
\begin{eqnarray}\label{os1}
&\le& \frac{1}{\log(1+\frac{1}{\sigma^2})} \log(1+ \frac{\sigma_{n-1}^2(s_n^{(n-1)})}{\sigma^2})\\\nn
&\le& \frac{2}{\log(1+\frac{1}{\sigma^2})}\bigg(I_n(\{s^{(n-1)}_j\}_{j=1}^n) 
% \\\nn
% &&\hspace{3em}
- I_{n-1}(\{s^{(n-1)}_j\}_{j=1}^{n-1})  \bigg) \\\nn
&=& \frac{2}{\log(1+\frac{1}{\sigma^2})}\bigg(
I_n(\{s^{(n-1)}_j\}_{j=1}^n) -
I_{n}(\{s^{(n)}_j\}_{j=1}^{n})\\\nn
&&\hspace{1em}+ I_{n}(\{s^{(n)}_j\}_{j=1}^{n})
-I_{n-1}(\{s^{(n-1)}_j\}_{j=1}^{n-1}) 
\bigg)
\end{eqnarray}}}
Inequality~\eqref{os1} holds since $z^2\le \frac{1}{\log(1+\frac{1}{\sigma^2})}\log(1+\frac{z^2}{\sigma^2})$ for all $z\le 1$ and $\sigma_{n-1}^2(s_n^{(n-1)}) \le1$. 

Summing both sides over $n$ from $1$
 to $N$ we get 
{\small{\begin{eqnarray}\nn
V_N
&\le& \frac{2}{\log(1+\frac{1}{\sigma^2})}\bigg(\
\sum_{n=1}^N 
I_n(\{s^{(n-1)}_j\}_{j=1}^n) -
I_{n}(\{s^{(n)}_j\}_{j=1}^{n}) \bigg)\\\label{eq:regd}
&+& \frac{2}{\log(1+\frac{1}{\sigma^2})} I_N(\{s^{(N)}_j\}_{j=1}^{N}). 
\end{eqnarray} }}
\hspace{-1em}
The first term on the right hand side of~\eqref{eq:regd} captures the change in the information gain by the variation of the observation points. The second term is a scaled information gain.
 
Let $\bar{\K}^{(n-1)}$ and $\K^{(n)}$ denote the covariance matrices of $\{s^{(n-1)}_j\}_{j=1}^{n-1}$ and $\{s^{(n)}_j\}_{j=1}^{n}$, respectively.
Since the maximum replacement of $\s_j$ from iteration $n-1$ to iteration $n$ is $\Delta_n$, we have
\begin{eqnarray}
\max_{i,j}\frac{[I+\sigma^{-2}\bar{K}^{(n-1)}]_{i,j}}{[I+\sigma^{-2}{K}^{(n)}]_{i,j}} \le \exp(\Delta_n^2).
\end{eqnarray}

Thus, for the first term on the right hand side of~\eqref{eq:regd}, we have
{\small{\begin{eqnarray}\nn
&&\hspace{-2em}I_n(\{s^{(n-1)}_j\}_{j=1}^{n-1}) -
I_{n}(\{s^{(n)}_j\}_{j=1}^{n}) \\\nn
&=& \frac{1}{2}\log |I+\sigma^{-2}\bar{\K}^{(n-1)}| -\frac{1}{2}\log |I+\sigma^{-2}{\K}^{(n)}|\\\nn
&=& \frac{1}{2}\log\frac{|I+\sigma^{-2}\bar{\K}^{(n-1)}|}{|I+\sigma^{-2}{\K}^{(n)}|} \\\label{lae1}
&\le&\frac{1}{2}\log\exp(n\Delta_n^2)
=\frac{n\Delta_n^2}{2}
\le\frac{Cd}{2n}.
\end{eqnarray}}}
where the last inequality come from condition on $\Delta_n$.
Combining~\eqref{eq:regd} and~\eqref{lae1}, we get
{\small{\begin{eqnarray}\nn
V_N &\le& \frac{2}{\log(1+\frac{1}{\sigma^2})}\bigg( I_N(\{s^{(N)}_j\}_{j=1}^{N}) + \sum_{n=1}^N \frac{Cd}{2n}\bigg)\\\nn
&\le&
\frac{2}{\log(1+\frac{1}{\sigma^2})}\gamma_n
+\frac{Cd}{\log(1+\frac{1}{\sigma^2})}\log(n+1)\\
&\le& C_1\gamma_n + C_2\log(n+1),
\end{eqnarray}}}
where $C_1 = \frac{2}{\log(1+\frac{1}{\sigma^2})}$ and $C_2=\frac{Cd}{\log(1+\frac{1}{\sigma^2})}$.

\newpage
\onecolumn
\section*{Appendix A.3: Experiments test functions and regret curves}

\begin{figure*}[ht]
 \centering
   \includegraphics[trim=150mm 25mm 180mm 78mm, clip, width=.32\linewidth]{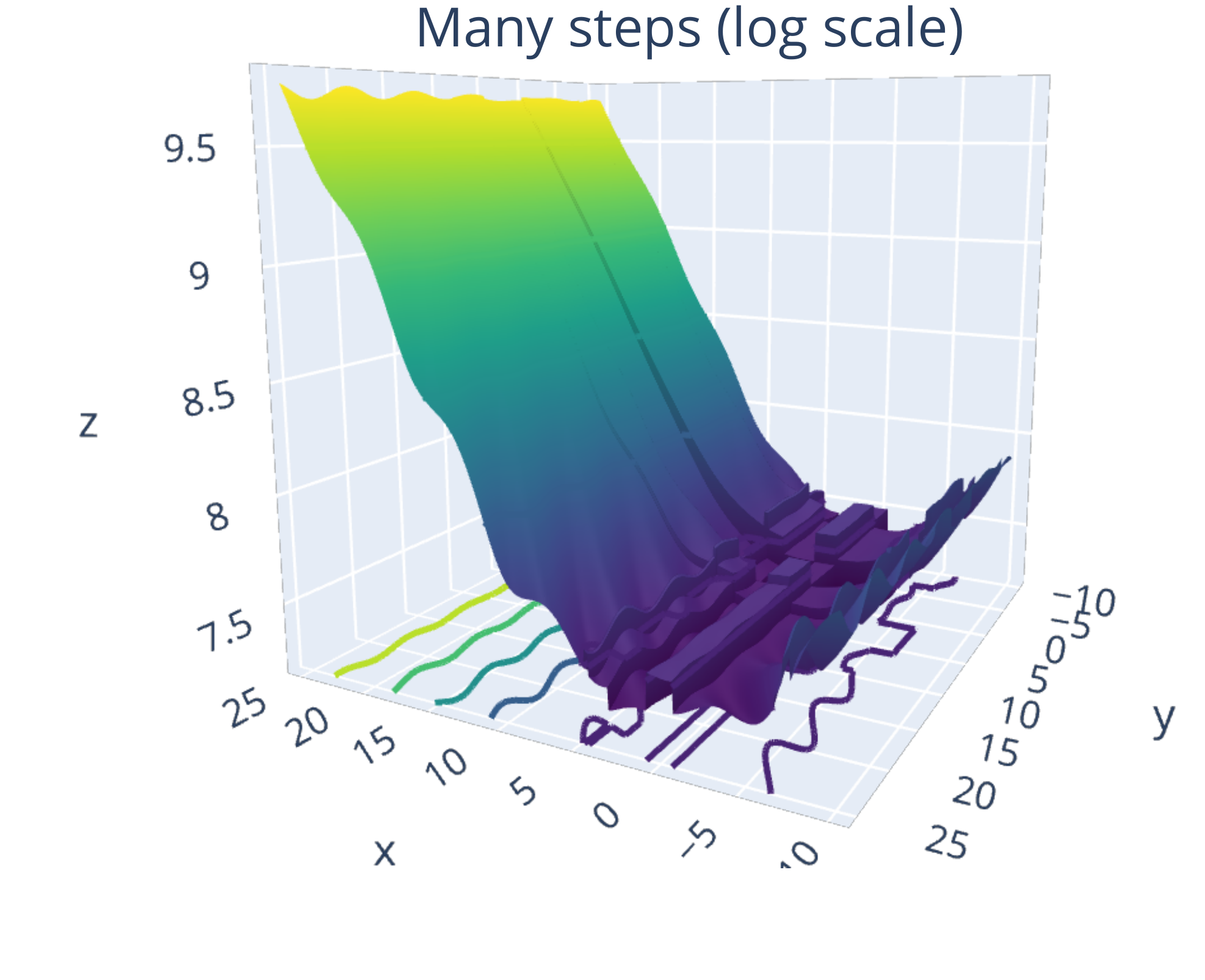}
   \includegraphics[trim=150mm 25mm 180mm 78mm, clip, width=.32\linewidth]{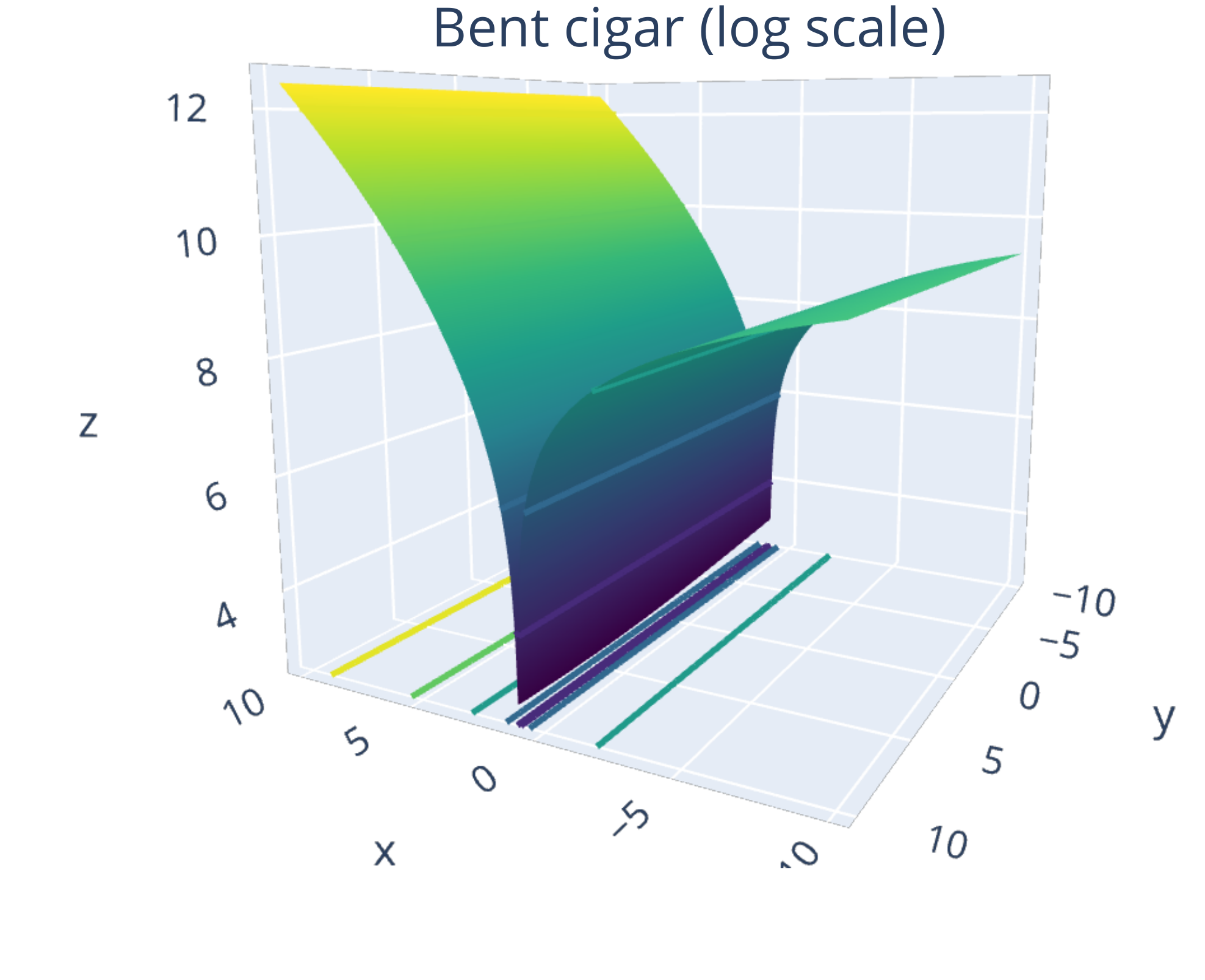}
  \includegraphics[trim=150mm 25mm 180mm 78mm, clip, width=.32\linewidth]{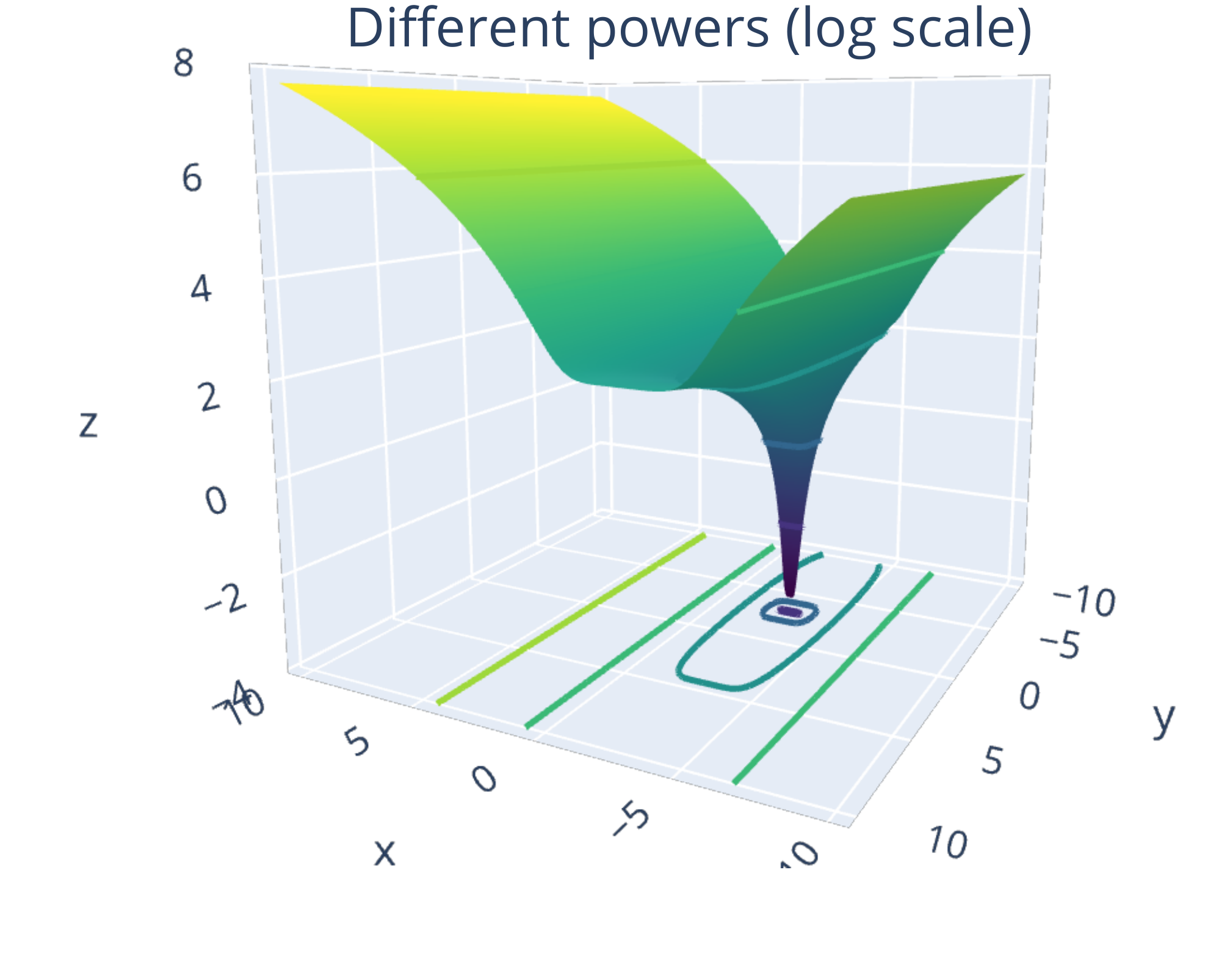}
   \caption{2D test functions.}\label{fig:multistep}
 \end{figure*}
 
 \begin{figure*}[ht]
  \includegraphics[trim=0mm 0mm 0mm 0mm, clip, width=.32\linewidth]{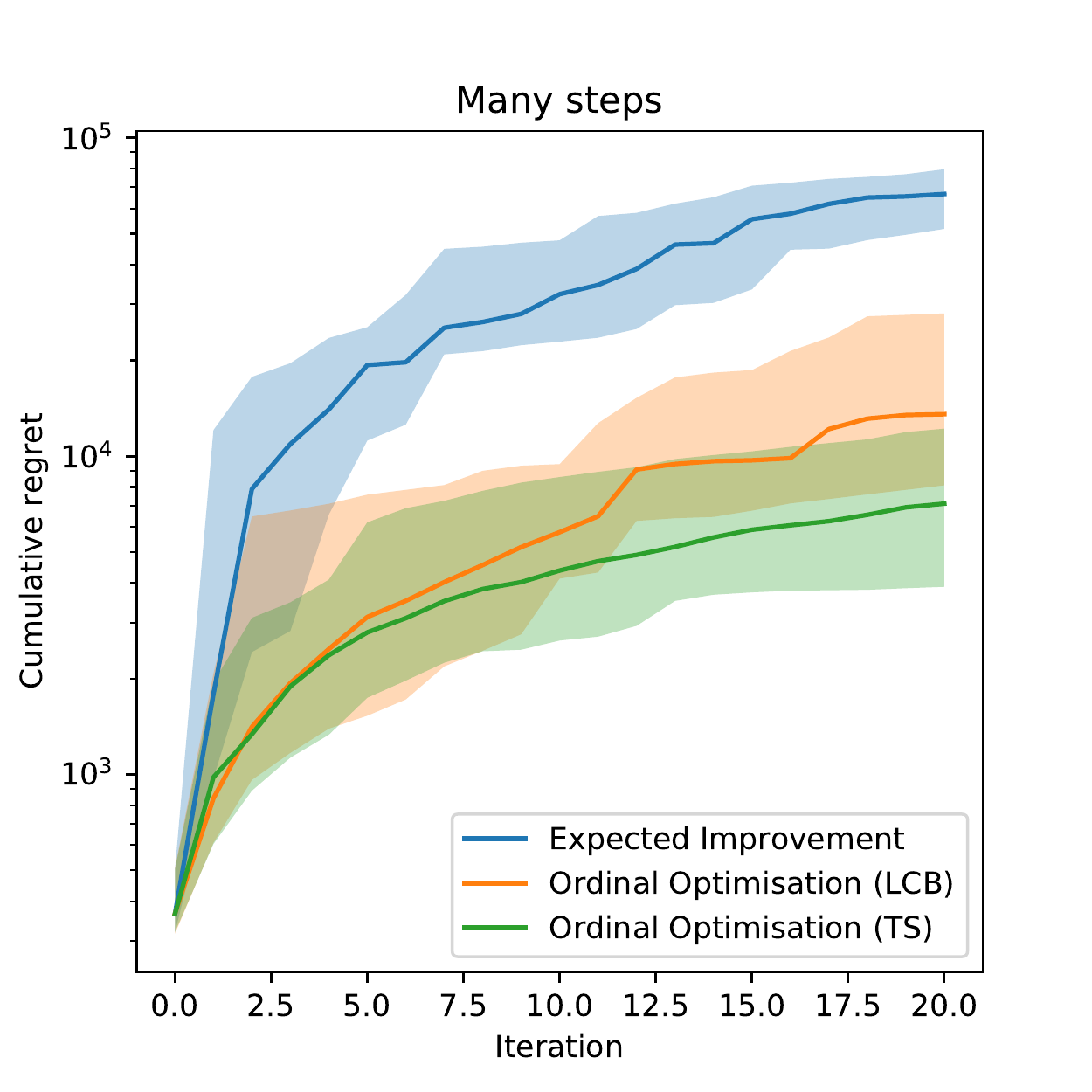}
\includegraphics[trim=0mm 0mm 0mm 0mm, clip, width=.32\linewidth]{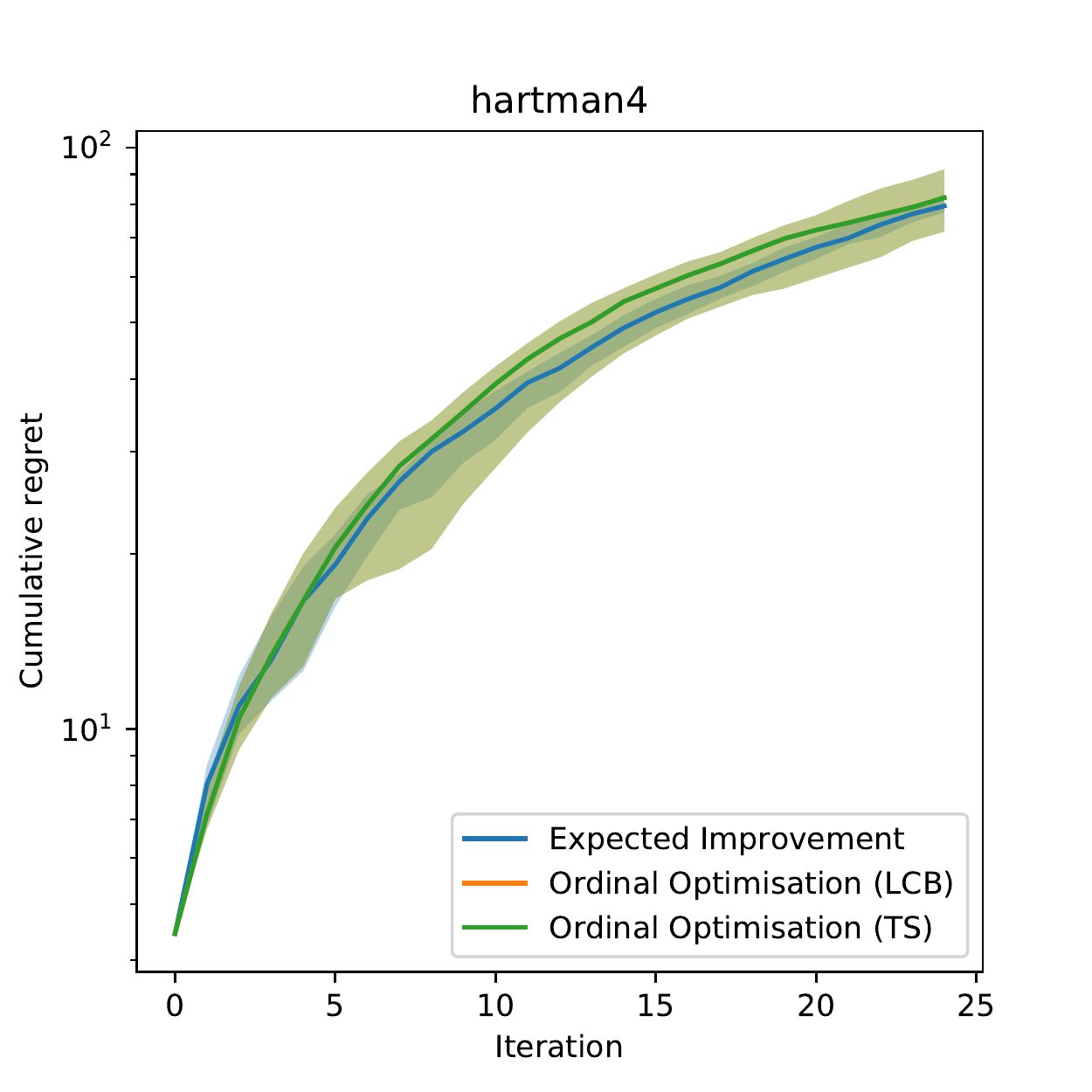}
  \includegraphics[trim=0mm 0mm 0mm 0mm, clip, width=.32\linewidth]{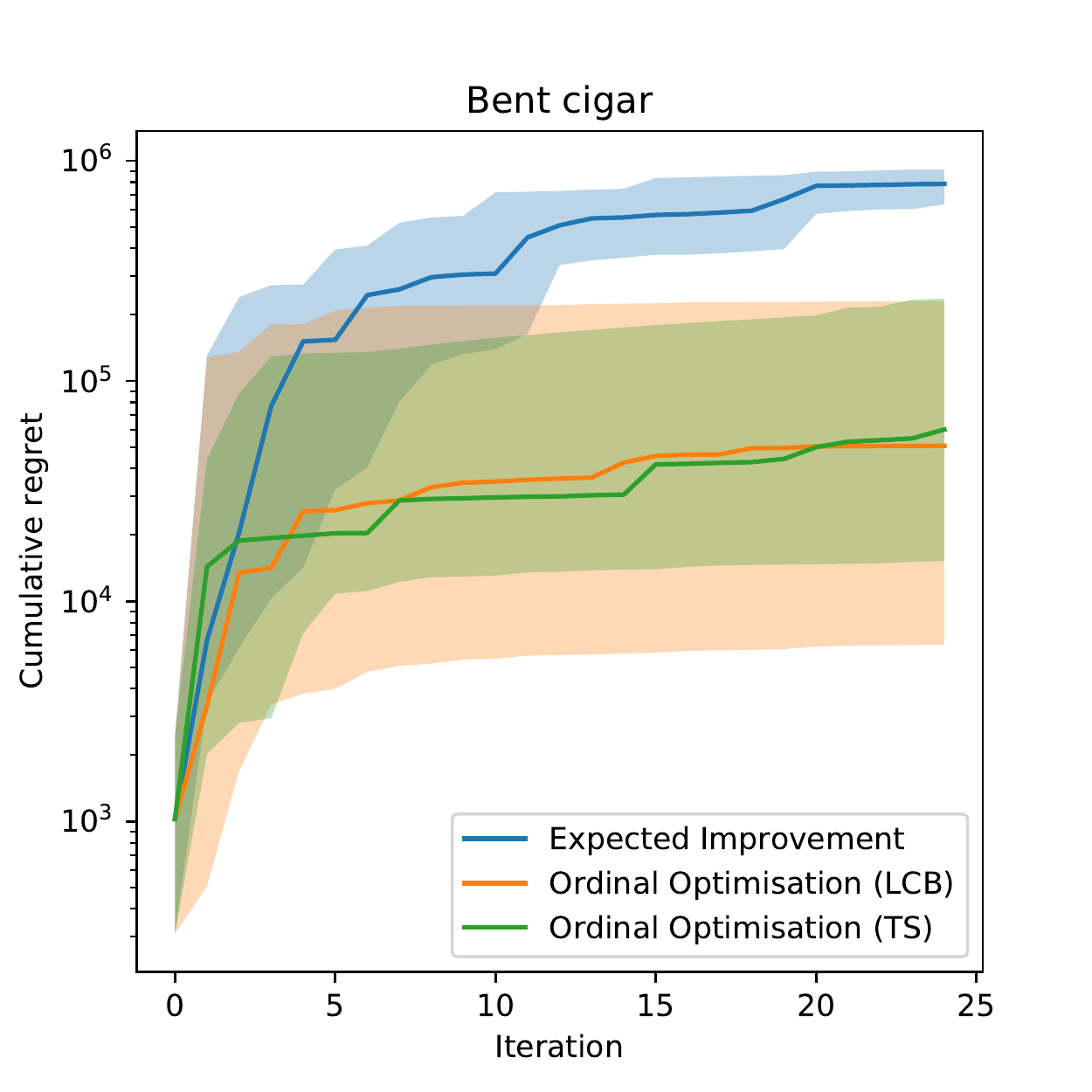}\\
  \includegraphics[trim=0mm 0mm 0mm 0mm, clip, width=.32\linewidth]{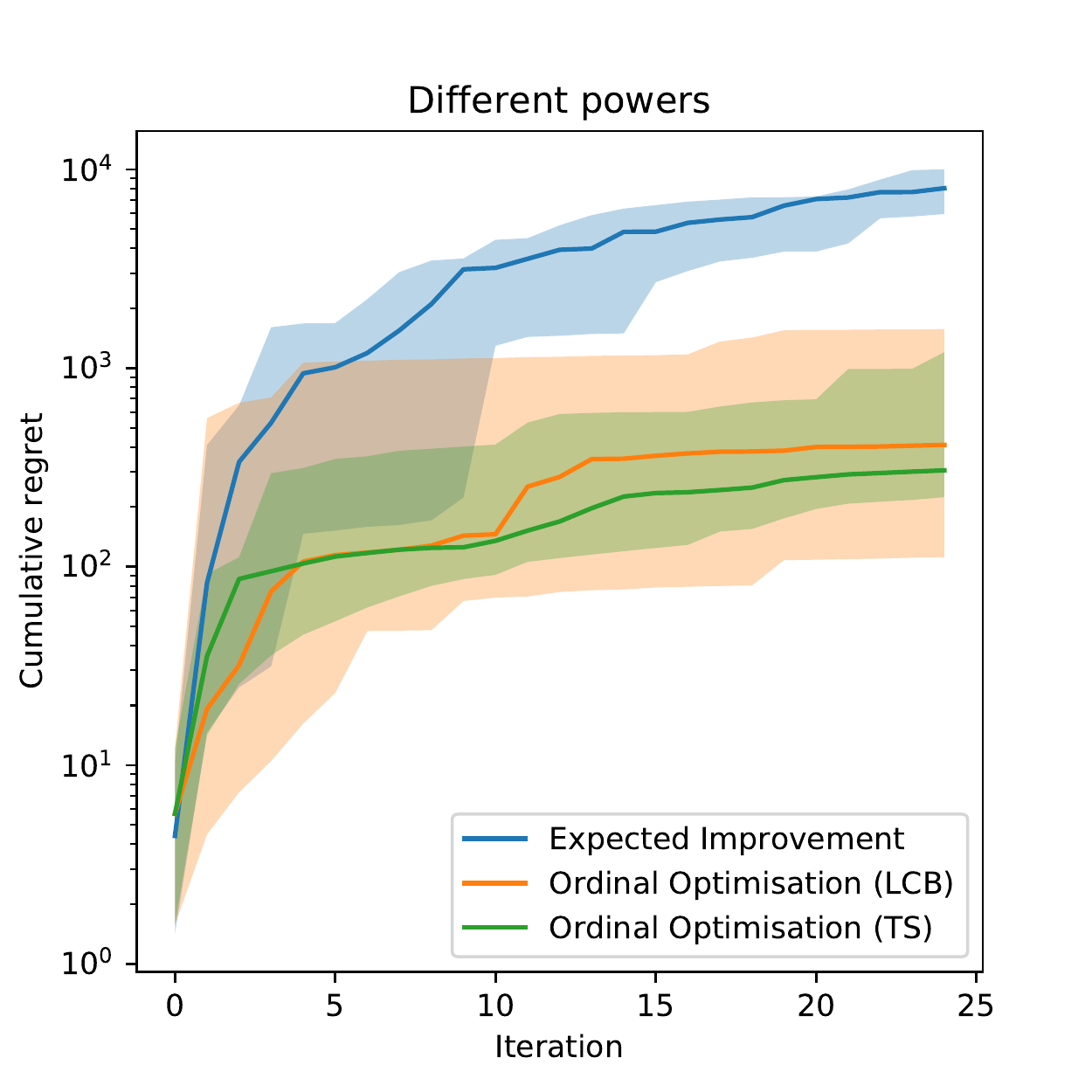}
  \includegraphics[trim=0mm 0mm 0mm 0mm, clip, width=.32\linewidth]{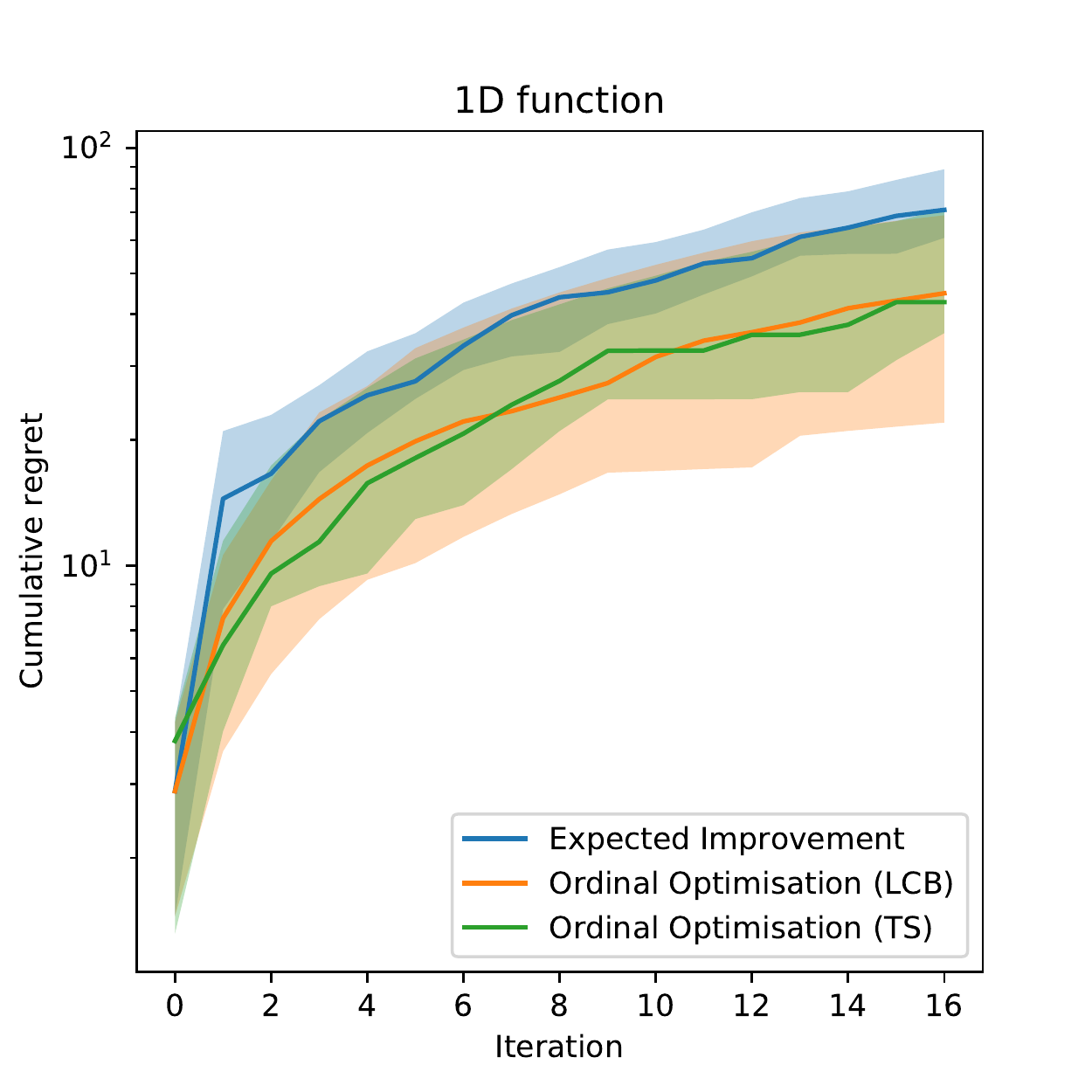}
   \caption{Cumulative regrets on the five test problems.}\label{fig:regret}
\end{figure*}

\twocolumn

\end{document}